\documentclass{article}

\usepackage{microtype}
\usepackage{graphicx}
\usepackage{booktabs} 

\usepackage[accepted]{icml2019}

\usepackage{verbatim}
\usepackage{amssymb}
\usepackage{amsthm}
\usepackage{color}
\usepackage{xr}
\usepackage{url}
\usepackage{bm}
\usepackage{bbm}
\usepackage{subcaption}
\usepackage{wrapfig}
\usepackage{afterpage}
\usepackage{MnSymbol}
\usepackage{algorithm}
\usepackage{algorithmic}
\usepackage{amsfonts}

\newtheorem{thm}{Theorem}[section]
\newtheorem{cor}[thm]{Corollary}
\newtheorem{prop}[thm]{Proposition}
\newtheorem{lem}[thm]{Lemma}
\newtheorem{definition}[thm]{Definition}

\newtheorem{rem}[thm]{Remark}
\usepackage{boxedminipage}

\renewcommand{\xi}{{\xx}^{(m)}}

\newcommand{\prob}{\mathbb{P}}

\newcommand{\cA}{{\cal A}}

\newcommand{\cB}{\mathcal{B}}

\newcommand{\cP}{\mathcal{P}}
\newcommand{\cL}{\mathcal{L}}
\newcommand{\cO}{\mathcal{O}}
\newcommand{\cW}{\mathcal{W}}

\newcommand{\cH}{\mathcal{H}}
\newcommand{\needcite}[1]{}

\newcommand{\cS}{{\cal S}}
\newcommand{\cX}{{\cal X}}

\newcommand{\be}{\begin{equation}}
\newcommand{\ee}{\end{equation}}
\newcommand{\benn}{\begin{equation*}}
\newcommand{\eenn}{\end{equation*}}
\newcommand{\bea}{\begin{eqnarray*}}
\newcommand{\eea}{\end{eqnarray*}}
\newcommand{\bean}{\begin{eqnarray}}
\newcommand{\eean}{\end{eqnarray}}

\newcommand{\ww}{\boldsymbol{w}} 
\newcommand{\xx}{\boldsymbol{x}}

\newcommand{\zz}{\boldsymbol{z}}
\newcommand{\thth}{\boldsymbol{\theta}}

\newcommand{\vv}{\boldsymbol{v}}

\newcommand{\uu}{\boldsymbol{u}} 
\newcommand{\pp}{\boldsymbol{p}}
\newcommand{\oo}{\boldsymbol{o}}

\newcommand{\cD}{{\cal D}}

\renewcommand{\aa}{\boldsymbol{a}}

\newcommand{\ignore}[1]{}

\newcommand{\polyring}[1]{\reals\left[x_1,\ldots,x_n\right]}

{
\begin{center}
\begin{boxedminipage}{0.8\linewidth}
\begin{center}
\textbf{\texttt{#1}}
\end{center}
\rm
\begin{tabbing}
....\=...\=...\=...\=...\=  \+ \kill
} %
{\end{tabbing} 
\end{boxedminipage} \end{center} 
}

\renewcommand{\eqref}[1]{Eq.~\ref{#1}}

\newcommand{\figref}[1]{Fig.~\ref{#1}}

\newcommand{\epsmax}{\epsilon_{\textup{max}}}
\newcommand{\epsmin}{\epsilon_{\textup{min}}}

\newcommand{\reals}{\mathbb{R}}

\newcommand{\wvec}[1]{\ww^{(#1)}}
\newcommand{\wmat}[1]{W^{(#1)}}
\newcommand{\jit}[1]{j^{(#1)}}
\newcommand{\pwvec}[1]{\pp^{(#1)}}
\newcommand{\ai}[1]{a^{(#1)}}
\newcommand{\aai}[1]{\aa^{(#1)}}
\newcommand{\uvec}[1]{\uu^{(#1)}}

\newcommand{\cnnth}[1]{N_{\textup{CNN}}[#1]}
\newcommand{\fcth}[1]{N_{\textup{FC}}[#1]}

\newcommand{\cnn}{N_{\textup{CNN}}}
\newcommand{\hcnn}{\mathcal{H}_{\textup{CNN}}}

\newcommand{\layeralg}{\text{LW}_{\text{CNN}}}

\newcommand\numberthis{\addtocounter{equation}{1}\tag{\theequation}}

\DeclareMathOperator*{\argmax}{arg\,max}
\DeclareMathOperator*{\argmin}{arg\,min}

\newcommand{\commenteps}[1]{}

\newcommand{\beps}[2]{\cB_{#1}^{#2}}

\icmltitlerunning{An Optimization and Generalization Analysis for Max-Pooling Networks}

\begin{document}




\twocolumn[
    \icmltitle{An Optimization and Generalization Analysis for Max-Pooling Networks}




\begin{icmlauthorlist}
\icmlauthor{Alon Brutzkus}{tau}
\icmlauthor{Amir Globerson}{tau}
\end{icmlauthorlist}

\icmlaffiliation{tau}{The Blavatnik School of Computer Science, Tel Aviv University}

\icmlcorrespondingauthor{Alon Brutzkus}{alonbrutzkus@mail.tau.ac.il}

\icmlkeywords{Machine Learning, ICML}

\vskip 0.3in
]

\printAffiliationsAndNotice{}




\begin{abstract}
    Max-Pooling operations are a core component of deep learning architectures. In particular, they are part of most convolutional architectures used in machine vision, since pooling is a natural approach to pattern detection problems. However, these architectures are not well understood from a theoretical perspective. For example, we do not understand when they can be globally optimized, and what is the effect of over-parameterization on generalization. Here we perform a theoretical analysis of a convolutional max-pooling architecture, proving that it can be globally optimized, and can generalize well even for highly over-parameterized models. Our analysis focuses on a data generating distribution inspired by pattern detection problem, where a "discriminative" pattern needs to be detected among "spurious" patterns. We empirically validate that CNNs significantly outperform fully connected networks in our setting, as predicted by our theoretical results.
\end{abstract}

\section{Introduction}
\label{sec:intro}




Convolutional neural networks (CNNs) have achieved remarkable performance in various computer vision tasks \citep{krizhevsky2012imagenet, xu2015show, taigman2014deepface}. Such networks typically combine convolution and max-pooling layers, and can thus be used for detecting complex patterns in the input. In practice, CNNs typically have more parameters than needed to achieve zero train error (i.e., are overparameterized). Despite the potential problem of non-convexity in optimization and overfitting because of overparameterization, training these models with gradient based methods leads to solutions with low test error. Furthermore, overparameterized CNNs significantly outperform fully connected networks (FCNs) on classifying image data \citep{malach2020computational}. Thus, a key question immediately arises:

\textit{Why do overparameterized CNNs generalize well on image data and outperform FCNs?}

To the best of our knowledge, this question remains largely unanswered. We note that the question contains two significant challenges: the first is to show that minimization of the non-convex training loss leads to high training accuracy (where non-convexity is a result of both max-pooling and ReLU activations), and the other is that over-fitting is avoided despite over-parameterization. The latter challenge is known as the question of inductive bias of gradient descent \citep{ZhangBHRV16}, and understanding it is a key goal of deep learning theory.

In this work, we provide the first results which address the above question. We theoretically analyze learning a simplified pattern recognition task with overparameterized CNNs and overparameterized FCNs. We consider a CNN with a convolution layer, max pooling and fully connected layer and compare it to a one-hidden layer non-linear FCN. Figure \ref{fig:intro} shows an example of our setup. We summarize our contributions as follows:
\begin{enumerate}
    \item \textbf{Expressive Power of CNNs with max-pooling:} We prove a novel VC dimension lower bound in our setting which is exponential in $d$, the filter dimension of the CNN. This result implies that there exists ERM algorithms which have sample complexity which is exponential in $d$ in our setting.
    \item \textbf{Optimization and Generalization for learning CNNs with max-pooling:} We analyze learning overparamaterized CNNs with a layerwise gradient descent optimizer. We show that the algorithm converges to zero training loss and the learning has a sample complexity of $O(d)$. This is despite the above VC result, which shows that general ERM optimizers can potentially overfit. In our proof, we analyze the dynamics of training the first layer. We show that it induces a representation in the last layer which is separable with large margin and thus implies a good generalization guarantee.
    \item \textbf{Generalization of FCNs:} We apply recent results of \citet{brutzkus2018sgd} which show a generalization bound for overparameterized FC networks that is independent of the network size. We prove that in our setting, their bound can be at best $O(d^{2r})$ for $r \ge 1$, and can thus be much larger than the sample complexity we derive for the CNN.
    \item \textbf{Empirical Evaluation:} We empirically validate our theoretical results. We show that CNNs generalize well and significantly outperform FCNs in our setting as predicted by our theory. We empirically confirm that this holds also for several extensions of our setup.
\end{enumerate}

Our results make a significant headway on the challenging problem of understanding why overparameterized CNNs can generalize better than overparameterized FCNs on image classification tasks. In particular, to the best of our knowledge, we provide the first optimization and generalization results for overparameterized CNNs with max pooling.

\begin{figure*}[t!]
	\begin{subfigure}{.46\textwidth}
		\includegraphics[width=1.0\linewidth,height=4.7cm]{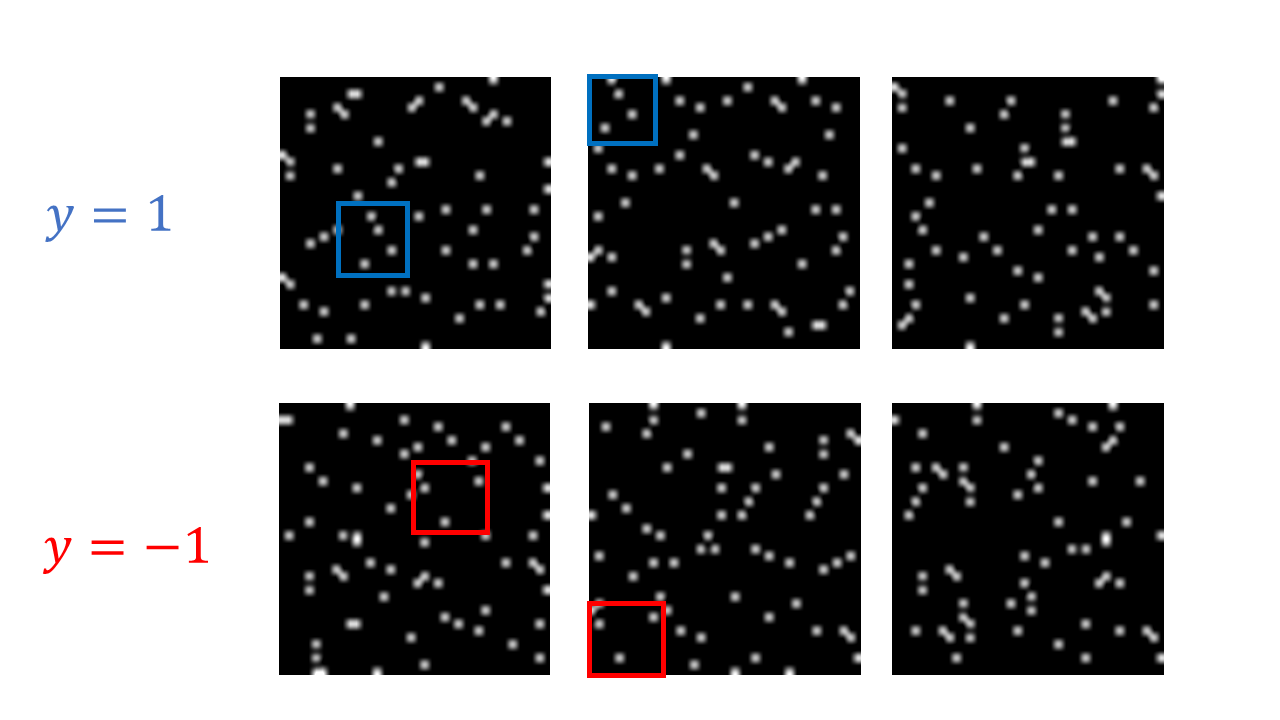}
		\caption{}
		\label{fig:mnist_data}
	\end{subfigure}\hspace{5mm}
	\begin{subfigure}{.46\textwidth}
		\includegraphics[width=1.0\linewidth]{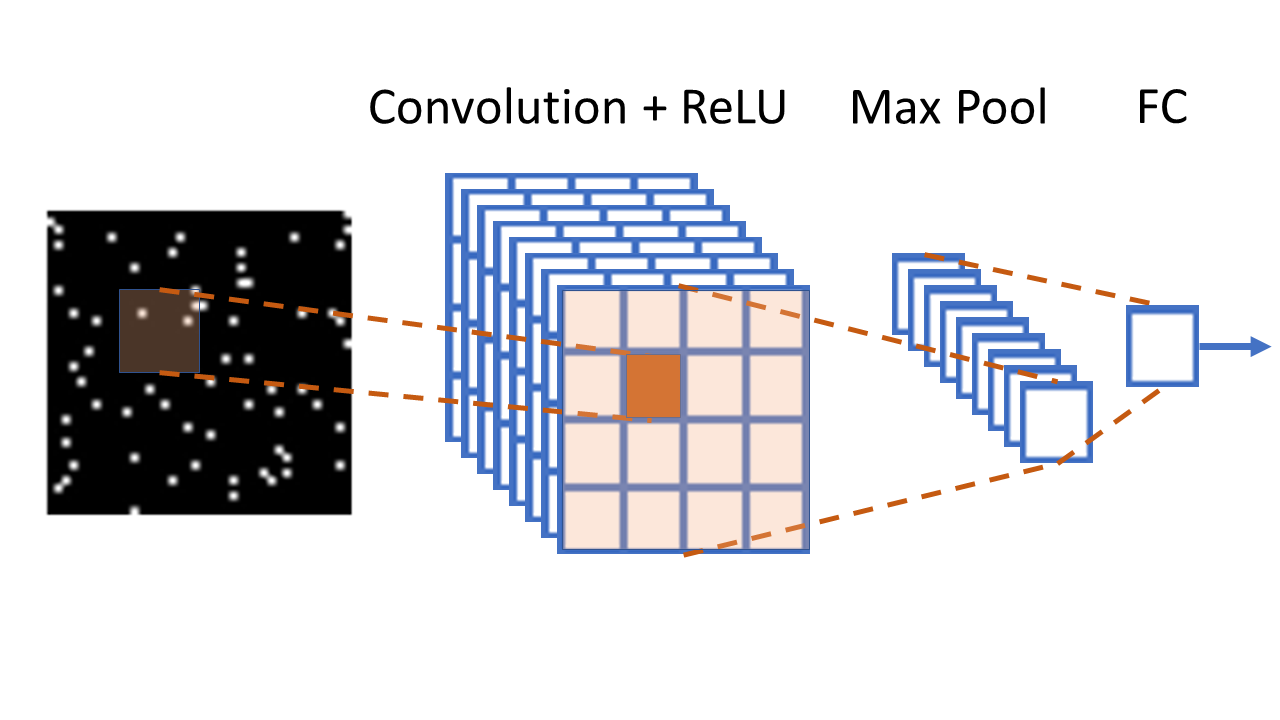}
		\caption{}
		\label{fig:mnist_filters}
	\end{subfigure}%
	\caption{(a) An example of the pattern detection tasks we consider. Input images consist of 4 rows of 4 patches each. Each image consists of a discriminative pattern. All other patterns are spurious and may appear in both classes. In the two leftmost images of each class, the corresponding discriminative pattern is shown. (b) An illustration of the architecture of the 3-layer overparameterized CNN we analyze in our setting.}
	\label{fig:intro}
\end{figure*}

\section{Related Work}
\label{sec:related_work}



Two recent works have provided theoretical support that that CNNs outperform FCNs. \citet{li2020convolutional} consider a simplified image classification task and prove a sample complexity gap between FCNs and \textit{single} channel CNNs. \citet{malach2020computational} prove that for simplified pattern detection tasks, there is a \textit{computational} separation between overparameterized CNNs and FCNs. Their generalization bound for overparameterized CNNs depends on the number of channels of the CNN. Therefore, both works do not show that over-parameterized CNNs are reslient to over-fitting, which is the main focus of our work.



Several recent works have studied the generalization properties of overparameterized CNNs. Some of these propose generalization bounds that depend on the number of channels \citep{long2019size, jiang2019fantastic}. Others provide guarantees for CNNs with constraints on the weights \citep{zhou2018understanding, li2018tighter}. 
Convergence of gradient descent to KKT points of the max-margin problem is shown in \cite{lyu2019gradient} and \cite{nacson2019lexicographic} for homogeneous models. However, their results do not provide generalization guarantees in our setting. \citet{gunasekar2018implicit} study the inductive bias of linear CNNs.

\citet{yu2019learning} study a pattern classification problem similar to ours. However, their analysis  their sample complexity guarantee depends on the network size, and thus does not explain why large CNNs do not overfit. Other works have studied learning under certain ground truth distributions. For example, \citet{brutzkus2019larger} study a simple extension of the XOR problem, showing that overparameterized CNNs generalize better than smaller CNNs. Single-channel CNNs are analyzed in \citep{du2017convolutional, du2017gradient, brutzkus2017globally, du2018many}. CNNs were analyzed via the NTK approximation \citep{li2019enhanced, arora2019exact}. Our analysis does not assume the NTK approximation. For example, we require a mild overparameterization in our results which does not depend on the number of samples, in contrast to NTK analyses. Furthermore, our results hold for sufficiently small initialization, which is not the regime of NTK analysis.

Other works study the inductive bias of gradient descent on fully connected linear or non-linear networks \citep{ji2018gradient, arora2019implicit, wei2019regularization, brutzkus2018sgd, dziugaite2017computing, allen2019learning, chizat2020implicit}. Fully connected networks were also analyzed via the NTK approximation \citep{du2018gradient2, du2018gradient, arora2019fine, fiat2019decoupling}. \citet{kushilevitz1996learning, shvaytser1990learnable} study the learnability of visual patterns distribution. However, our focus is on learnability using a specific algorithm and architecture: gradient descent trained on overparameterized CNNs.



\section{Preliminaries}
\label{sec:problem_formulation}

\paragraph{\underline{Data Generating Distribution:}} 
We consider a learning problem that captures a key property of visual classification. Many visual classes are characterized by the existence of certain patterns. For example an 8 will typically contain an x like pattern somewhere in the image. Here we consider an abstraction of this behavior where images consist of a set of patterns. Furthermore, each class is characterized by a pattern that appear exclusively in it. We define this formally below. 

Let $\cO$ be a set of $3\le l\le d$ orthogonal vectors in $\reals^{d}$. For simplicity, we assume that $\left\|\oo\right\|_2 = 1$ for all $\oo \in \cO$. We denote $\cO = \left\{\oo_1,\oo_2,...,\oo_l\right\}$ and use the notation $i \in \cO$ to denote $\oo_i \in \cO$. 

\commenteps{
For each $i \in \cO$ we define the ball of radius $\epsilon_i> 0$ around $\oo_i$ by $\beps{i}{\epsilon_i} = \left\{\pp \in \reals^d \mid \left\|\pp - \oo_i\right\| \le \epsilon_i \right\}$, where $\left\|\cdot\right\|$ is the Euclidean norm.  We refer to any vector in $\reals^d$ as a \textit{pattern}. We define $\epsmax = \max_i{\epsilon_i}$ and $\epsmin = \min_i{\epsilon_i}$. 
}

We consider input vectors $\xx$ with $n$ patterns. Formally, $\xx=(\xx[1],...,\xx[n]) \in \reals^{n d}$
where $\xx[i] \in \reals^d$ is the $i$th pattern of $\xx$.\footnote{We will generally use the notation $\vv[i] \in \reals^d$ for  any vector $\vv \in \reals^{nd}$.} We say that $\xx$ contains $\pp$ if there exists $j$ such that $\xx[j]=\pp$. We denote $\pp \in \xx$ if $\xx$ contains the pattern $\pp \in \reals^d$.  Let $\cP({\xx}) = \left\{\pp \in \xx \mid \pp \in \reals^d\right\}$ denote the set of all patterns in $\xx$. 

\commenteps{
Next, we define how labeled points are generated by a distribution parameterized by $\epsilon_1,...,\epsilon_l$. In our setting we consider three types of patterns: positive, negative and spurious. We will refer to the patterns in the set $\beps{1}{\epsilon_1}$ as positive, the patterns in the set $\beps{2}{\epsilon_2}$ as negative and the patterns in any set $\beps{i}{\epsilon_i}$, $3 \le i \le l$, as spurious. We let $\cS = \left\{3,...,l\right\}$.
}
Next, we define how labeled points are generated. In our setting we consider three types of patterns: positive, negative and spurious. We will refer to the pattern $\oo_1$ as positive, the pattern $\oo_2$ as negative and the patterns $\oo_3,\ldots,\oo_l$ as spurious. We let $\cS = \left\{3,...,l\right\}$.

\commenteps{
We consider distributions $\cD$ over $\left(\xx, y\right) \in \reals^{n d} \times \{\pm 1\}$. We assume that for each $1 \le i \le l$ there is a distribution $\mu_i$ supported on $\beps{i}{\epsilon_i}$. In the distribution $\cD$, each positive sample contains a single positive pattern and $n-1$ randomly sampled spurious patterns. Similarly, a negative sample has a single negative pattern and $n-1$ spurious patterns. Formally, we define $\cD$ with the following properties:
\begin{enumerate}
    \item $\prob\left(y = 1\right) = \prob\left(y = -1\right) = \frac{1}{2}$.
    \item Given $y=1$, a vector $\xx$ is sampled as follows. Randomly sample a pattern $\pp_+ \in \beps{1}{\epsilon_1}$ with respect to $\mu_1$. Randomly sample an index $1 \le j_+ \le n$ and set $\xx\left[j_+\right] = \pp_+$. Then, for each $1 \le j \le n$ such that $j \neq j_+$, randomly choose $i_j \in \cS$ and sample a vector $\pp_j \in \beps{i_j}{\epsilon_{i_j}}$ with respect to $\mu_{i_j}$. Set $\xx[j] = \pp_j$.
    \item Given $y=-1$, sample a pattern $\pp_- \in \beps{2}{\epsilon_2}$ with respect to $\mu_2$ and set $\xx\left[j_-\right] = \pp_-$ for a randomly sampled $1 \le j_- \le n$. The remaining $n-1$ patterns are spurious sampled exactly as in the case of $y=1$.
\end{enumerate}
}

We consider distributions $\cD$ over $\left(\xx, y\right) \in \reals^{n d} \times \{\pm 1\}$. In the distribution $\cD$, each positive sample contains the positive pattern and $n-1$ randomly sampled spurious patterns. Similarly, a negative sample has a single negative pattern and $n-1$ spurious patterns. Formally, we define $\cD$ with the following properties:
\begin{enumerate}
    \item $\prob\left(y = 1\right) = \prob\left(y = -1\right) = \frac{1}{2}$.
    \item Given $y=1$, a vector $\xx$ is sampled as follows. Randomly sample an index $1 \le j_+ \le n$ for placing the positive pattern, and set $\xx\left[j_+\right] = \oo_1$. Then, for each $1 \le j \le n$ such that $j \neq j_+$, randomly choose $i_j \in \cS$ and set $\xx[j] = \oo_{i_j}$.
    \item Given $y=-1$, do the same as $y=1$, using $\oo_2$ instead of $\oo_1$.
\end{enumerate}


\figref{fig:mnist_data} shows an example of the above distribution $\cD$. 

\paragraph{\underline{CNN Architecture:}}
For learning the above distributions, we consider a 3-layer CNN that consists of a convolutional layer with non-overlapping filters, followed by ReLU, max pooling and a fully-connected layer. The network is parametrized by $\thth = \left(W,\aa\right)$ where $W\in \reals^{k \times n}$ and each row $i$ of $W$, denoted by $\ww_i \in \reals^d$, corresponds to a different channel. The vector $\aa=(a_1,...,a_k) \in \reals^k$ corresponds to the weights of the fully connected layer. 

For an input $\xx=(\xx[1],...,\xx[n]) \in \reals^{n d}$ where $\xx[i] \in \reals^d$, the output of the network is:
\begin{equation}
\label{eq:network}
\cnnth{\thth}(\xx)=
\sum_{i=1}^{k}a_i\Big[\max_j\left\{\sigma\left(\ww_i\cdot \xx[j] \right)\right\}\Big] 
\end{equation}
where $\sigma(x)=\max\{0,x\}$ is the ReLU activation. For simplicity, we will usually denote $\cnn(\xx)$ when $\thth$ is clear from the context. We define $\hcnn(\cX)$ to be the hypothesis class of all functions $\textup{sign}\left(\cnn\right):\cX \rightarrow \{\pm 1\}$, where $\cX \subseteq \reals^{nd}$.\footnote{We assume WLOG that $\textup{sign}(0)=-1$. Furthermore, we note that the network $\cnn$ can have any number of channels $k$.} 

\paragraph{\underline{CNN Training Algorithm:}}

For the analysis of learning CNNs, we will consider a layerwise optimization algorithm which performs gradient updates layer-by-layer, starting from the first layer. Layerwise optimization algorithms are used in practice and have been shown to achieve performance that is comparable to end-to-end methods, e.g., on ImageNet \citep{belilovsky2019greedy}. Furthermore, the assumption on layerwise optimization has been used previously for theoretically analyzing neural networks \citep{malach2020implications}.

 For a set of points $A \subseteq \reals^{nd} \times \{\pm 1\}$ we consider minimizing the loss: 
 \begin{equation}
\label{eq:loss}
\cL[A](\thth) = \frac{1}{\left|A\right|}\sum_{(\xx,y) \in A}\ell\left(y\cnnth{\thth}(\xx)\right)
 \end{equation}
where $\ell(x) = \log\left(1+e^{-x}\right)$ is the binary cross entropy loss. Let $S = \left\{(\xx_1,y_1),...,(\xx_m,y_m)\right\}$ be a training set with $m$ IID samples from $\cD$. For the analysis, we partition $S = S_1 \cup S_2$ to two disjoint sets $S_1$ and $S_2$ such that $S_1 = \left\{(\xx_1,y_1),...,(\xx_{\lceil\frac{m}{2}\rceil},y_{\lceil\frac{m}{2}\rceil})\right\}$. We denote $\cL_i = \cL[S_i]$, and $m_i = \left|S_i\right|$ for $i\in \{1,2\}$. For convenience, we will say that $\xx \in S_i$ if there exists $y \in \{\pm 1\}$ such that $(\xx,y) \in S_i$. We denote the set of positive samples in $S_i$ by $S_i^+ = \left\{\xx \mid (\xx,1) \in S_i\right\}$ and the negative samples in $S_i$ by $S_i^- = \left\{\xx \mid (\xx,-1) \in S_i\right\}$.

The layerwise optimization algorithm for learning CNNs is given in Figure \ref{fig:layerwise}. The reason we optimize over two losses is technical: we need a fresh IID sample ($S_2$) in the second layer optimization for the generalization analysis (see Section \ref{sec:gen_conv}).

We define $\wvec{t}_i$ to be the $i$th row of $\wmat{t}$. For $\xx \in S$, $t > 0$ and $1 \le i \le k$, define $\jit{t}_i(\xx) = \argmax_{1\le j\le n}{\wvec{t}_i\cdot \xx[j]}$, i.e., $\jit{t}_i(\xx)$ corresponds to the pattern in $\xx$ that maximally activates $\wvec{t}_i$. If $\wvec{t}_i \cdot \xx\left[\jit{t}_i(\xx)\right] > 0$, define $\pwvec{t}_i(\xx) = \xx\left[\jit{t}_i(\xx)\right]$. Otherwise, define $\pwvec{t}_i(\xx) = 0$. Notice that the following equality holds: 
\begin{equation}
\label{eq:max_pattern}
     \max_j\left\{\sigma\left(\wvec{t}_i \cdot \xx[j] \right)\right\} = \wvec{t}_i \cdot \pwvec{t}_i(\xx)
\end{equation}
\begin{rem}
We note that it is necessary to make assumptions regarding the data distribution because the general case is intractable for optimization (because it includes
neural net learning as a special case). We believe that our data generating distribution does reflect core aspects of pattern detection problems. Furthermore, 
the analysis of overparameterized max pooling networks has not been performed for any task, and analysis of simplified tasks has been shown to be fruitful for understanding CNNs \citep{li2020convolutional,malach2020computational}. Additionally, non-overlapping filters are used in practice, and multiple theoretical works have analyzed CNNs with non-overlapping filters due to their tractability \citep{sharir2018expressive}. Finally, we note that in Section \ref{sec:exps} we show that our analysis is in line with the performance of CNNs and FCNs in more complex tasks.

\end{rem}

\begin{figure}
	\begin{algorithm}[H]
		\caption{$\layeralg$}
		\label{alg:layerwise}
		\begin{algorithmic}
			\STATE \textbf{Input:} Training set $S \subseteq \reals^{n d} \times \{\pm 1\}$, numbers of iterations $T_1,T_2 \in \mathbb{N}$ and learning rates $\eta_1, \eta_2 \in \reals$.
			\STATE Initialize $\wmat{0}$ and $\ai{0}$.
			\STATE \textbf{for} $t=1,...,T_1$ \textbf{do}: 
			\begin{ALC@g}
				\STATE $\wmat{t} \leftarrow \wmat{t-1} - \eta_1\frac{\partial \cL_1}{\partial W}\left(\wmat{t-1},\aai{0}\right)$.
			\end{ALC@g}
			\STATE \textbf{for} $t=1,...,T_2$ \textbf{do}: 
			\begin{ALC@g}
				\STATE $\aai{t} \leftarrow \aai{t-1} - \eta_2\frac{\partial \cL_2}{\partial \aa}\left(\wmat{T_1},\aai{t-1}\right)$.
			\end{ALC@g}
		\STATE \textbf{return} $\left(\wmat{T_1},\aai{T_2}\right)$.
		\end{algorithmic}
	\end{algorithm}
	\caption{Layerwise optimization algorithm for CNNs.}
	\label{fig:layerwise}
\end{figure}

\section{VC Dimension Bound}
\label{sec:vc}
Thus far we described a data generating distribution and a neural architecture. We now ask how expressive is this neural architecture. Because of the pooling layer, it may seem that the network has limited capacity, even for an unbounded number of channels. However, as we show next the capacity in terms of VC dimension is in fact exponential in $d$ in this case. This in turn means that the network can separate datasets of size up to exponential in $d$, and can thus potentially overfit badly. As we show in later sections, overfitting is avoided when learning using gradient descent. 

Fix $\cX \subseteq \reals^{nd}$ to be the support of the distribution $\cD$, i.e., each input vector consist of either a positive or negative pattern and $n-1$ spurious patterns. Denote the VC dimension of $\hcnn(\cX)$ by $\textup{VCdim}\left(\hcnn(\cX)\right)$. If we find $\textup{VCdim}\left(\hcnn(\cX)\right)$, then we can apply generalization bounds which show that any Empirical Minimization algorithm (ERM) has sample complexity of $O\left(\textup{VCdim}\left(\hcnn(\cX)\right)\right)$ \citep{blumer1989learnability}, and there exists ERMS with a tight lower bound.\footnote{Recall that an ERM algorithm is any algorithm which minimizes the empirical risk. See \citet{shalev2014understanding} for details.} Thus, lower bounding the VC dimension leads to worst-case lower bound on sample complexity,


We begin by recalling the definition of the VC dimension.
\begin{definition}
Let $\cH$ be a hypothesis class of functions from $\cX$ to $\{\pm 1\}$. For any non-negative in integer $m$, we define:
\begin{equation}
    \Pi_{\cH}(m) = \max_{x_1,...,x_m\in\cX}\left|\left\{\left(h(x_1),...,h(x_m)\right) \mid h \in \cH\right\}\right|
\end{equation}
If $\left|\left\{\left(h(x_1),...,h(x_m)\right) \mid h \in \cH\right\}\right| = 2^m$, we say that $\cH$ shatters the set $\left\{x_1,...,x_m\right\}$. The VC dimension of $\cH$, denoted by, $\textup{VCdim}\left(\cH\right)$, is the size of the largest shattered set, or equivalently, the largest $m$ such that $\Pi_{\cH}(m) = 2^m$.
\end{definition}
\commenteps{
By the definition above, we see that the largest shattered set determines the VC dimension. In our setting, any set of samples is a subset of the support of the distribution $\cD$. Therefore, we should set $\cX$ to be the support of $\cD$. We note that in more general analyses of the VC dimension of neural netwotks, e.g. in \citet{bartlett2019nearly}, $\cX$ is usually set to be the whole real input space (e.g., in our case $\reals^{nd}$). However, if we assume this, we will get a less accurate lower bound, because then we will consider shattered sets which are impossible to sample in our setting.

Therefore, after fixing $\left\{\epsilon_i\right\}_i$, we set $\cX$ to be the set of $\xx \in \reals^{nd}$, such that $\xx$ can contain only one pattern from either $\beps{1}{\epsilon_1}$ or $\beps{2}{\epsilon_2}$ and all other patterns are from $\beps{i}{\epsilon_i}$ for $1 \le i \le l$. For the analysis in this section we also assume that $l=d$. What is $\textup{VCdim}\left(\hcnn(\cX)\right)$ in this case? 
}

In the next theorem we show that $\textup{VCdim}\left(\hcnn(\cX)\right)$ is at least exponential in $d$. Therefore, the best generalization bound we can hope for using a VC dimension analysis scales exponentially with $d$.\footnote{By fixing $\cX$ to be the support of $\cD$ we get a more accurate VC lower bound than the case where $\cX = \reals^{nd}$. This is because in the latter case, shattered sets that are impossible to sample from $\cD$ may be considered in the lower bound.}


\begin{thm}
\label{thm:vc}
Assume that $d=2n$ and $n \ge 2$, then $\textup{VCdim}\left(\hcnn(\cX)\right) \ge 2^{\frac{d}{2}-1}$.
\end{thm}
\begin{proof}
We will construct a set $B \subseteq \cX$ of size $2^{n-1} = 2^{\frac{d}{2}-1}$ that can be shattered. We note that the inclusion $B \subseteq \cX$ will hold for any $\epsilon_i > 0$, $1 \le i \le d$. For a given $I \in \{0,1\}^{n-1}$ let $I[j]$ be its $j$th entry. For any such $I$, define a point $\xx_I$ such that for any $1 \le j \le n-1$, $\xx_I[j] = I[j]\oo_{2j+1} + (1-I[j])\oo_{2j+2}$. Furthermore, arbitrarily choose $\xx_I[n] = \oo_{1}$ or $\xx_I[n] = \oo_{2}$ and define $B = \left\{\xx_I \mid I \in \{0,1\}^{n-1}\right\}$.

Now, assume that each point $\xx_I \in B$ has label $y_I$. We will show that there is a network $\cnn \in \hcnn(\cX)$ such that $\cnn(\xx_I) = y_I$ for all $I$. For each $I \in \{0,1\}^{n-1}$, define $\wvec{I} = \max\left\{\alpha_I,0\right\}\sum_{1 \le j \le n-1}\xx_I[j]$ and $\uvec{I} = \max\left\{-\alpha_I,0\right\}\sum_{1 \le j \le n-1}\xx_I[j]$, where $\left\{\alpha_I\right\}$ is the unique solution of the following linear system with $2^{n-1}$ equations. For each $I \in \{0,1\}^{n-1}$ the system has the following equation:
\begin{align*}
\label{eq:linear_equations_vc}
    \sum_{I' \in \{0,1\}^{n-1} \setminus \{I\}}\alpha_{I'} = y_{I^c} \numberthis
\end{align*}
where for any $I \in \{0,1\}^{n-1}$, $I^c \in \{0,1\}^{n-1}$ is defined such that $I^c[j] = 1-I[j]$ for all $1 \le j \le n-1$.
There is a unique solution because the corresponding matrix of the linear system is the difference between an all 1's matrix and the identity matrix. By the Sherman-Morrison formula \citep{sherman1950adjustment}, this matrix is invertible, where in the formula the outer product rank-1 matrix is the all 1's matrix and the invertible matrix is minus the identity matrix. 

Set $W$ to be the matrix with rows $\wvec{I}$ followed by rows $\uvec{I}$. Let $\aa$ be the a vector of dimension $2^n$ such that $\aa = (\underbrace{1,...,1}_{2^{n-1}},\underbrace{-1,...,-1}_{2^{n-1}})$.
 
Then, for $\cnn$ with parameters $\thth = \left(W,\aa\right)$ and any $\xx_I$:

\begin{align*}
    &\cnnth{\thth}(\xx_I) = \sum_{I' \in \{0,1\}^{n-1}}  \Big[\max_j\left\{\sigma\left(\wvec{I'}\cdot \xx[j] \right)\right\} \\ &- \max_j\left\{\sigma\left(\uvec{I'} \cdot \xx[j] \right)\right\}\Big] \\
    &= \sum_{I' \in \{0,1\}^{n-1}}{\alpha_{I'} \max_j\left\{\sigma\left(\sum_{1 \le i \le n-1}\xx_{I'}[i] \cdot \xx_I[j] \right)\right\}} \\ &= \sum_{I' \in \{0,1\}^{n-1} \setminus \{I^c\}}\alpha_{I'} = y_{I}
\end{align*}
by the definition of $\cnn$, the orthogonality of the patterns $\{\oo_i\}_i$, and \eqref{eq:linear_equations_vc}. We have shown that any labeling $y_I$ can be achieved, and hence the set is shattered, completing the proof.
\end{proof}

The main limitation of the VC analysis is that it does not take into account the specific implementation of the ERM algorithm \citep{zhou2018understanding}. In the next section, we will show a more fine-grained analysis which is specific to the layerwise optimization algorithm, and can thus benefit from the specific inductive bias of this algorithm. As a result, we will obtain a significantly better generalization guarantee.

\section{Generalization Analysis of Gradient Descent} 
\label{sec:gen_conv}
In this section we analyze the optimization and generalization performance of the layer-wise gradient descent algorithm $\layeralg$ for training overparameterized CNNs (\eqref{eq:network}). We will show that it converges to zero training loss and its sample complexity is $O(d)$. This is in contrast to the result of the previous section which shows a VC dimension lower bound which is exponential in $d$, and therefore there are other ERM algorithms that can result in arbitrarily bad test error. 

For simplicity of the analysis, we assume that we initialize each filter $\wvec{0}_i$ from the $(d-1)$-sphere of radius $r$, namely, $\left\{\zz \in \reals^d \mid \left\|\zz\right\| = r\right\}$. We sample each $\ai{0}_i \in \reals$ uniformly at random from $\{\pm 1\}$. Additionally, the parameters $\wmat{0}$ and $\aai{0}$ are sampled independently. Our main result is summarized in the following theorem.

\begin{thm}
\label{thm:conv}
Let $S$ be an IID training set of size $m$ sampled from $\cD$. Assume that we run $\layeralg$ with $T_1 > 0$, $\eta_1 \le \frac{1}{4k(T_1+1)}$ and $\eta_2 < 8k$. Assume that $r \le \frac{\eta_1 }{200}$ and $k > 8d^3$. Then, with probability at least $(1- \delta)(1 - 4e^{-d} - 4e^{-\frac{m}{36}}$), the following holds:\footnote{The factor $e^{-d}$ in the confidence guarantee can be improved to $e^{-\Theta(k)}$. Note that the algorithm can be boosted with multiple restarts. We note also that $O(\cdot)$ hides a dependence on $\delta$.} \\ (1) $\lim_{T_2 \rightarrow \infty}\cL_2\left(\left(\wmat{T_1},\aai{T_2}\right)\right) = 0$. \\
(2) $\lim_{T_2 \rightarrow \infty}\prob_{(\xx,y) \sim \cD}\left(\textup{sign}\left(\cnnth{\left(\wmat{T_1},\aai{T_2}\right)}(\xx)\right)\neq y\right) \\ = O\left(\sqrt{\frac{d}{m}}\right)$
\end{thm}


The first part of the theorem is an optimization result stating that the $\layeralg$ will converge to zero $\cL_2$ loss. We note that this is despite the non-convexity of the loss $\cL_2$. The second part of the theorem states that the learned classifier will have a test error of order $\sqrt{\frac{d}{m}}$. Thus, the sample complexity is \textit{linear} in $d$. This is contrast to the VC dimension bound which is exponential in $d$.


Before proving the theorem, we make several remarks on the result. First, for simplicity we present asymptotic results for $T_2$. We can provide convergence rates that depend linearly on $d$ by changing the second layer optimization hyper-parameters (initialization and step size) and use recent results of \citet{ji2019refined}. See Section \ref{sec:convergence_rates} for details. Second, note that $k > 8d^3$ is a mild overparameterization condition, compared to other results which require $k$ to depend on the number of samples \citep{du2018gradient, ji2019polylogarithmic}. 

\begin{proof}[Proof of Theorem \ref{thm:conv}] 
We will prove the theorem in three parts. We defer the proofs of technical lemmas to the supplementary. We first outline the main ideas of the proof. In the first part we will prove a property of the initialization of the first layer. We show that at initialization there are sufficiently many ``lucky'' filters $\wvec{0}_i$ in the following sense. Either the pattern in $\cO$ that maximally activates them is $\oo_1$ and $\ai{0}_i = 1$, or the maximum activating pattern is $\oo_2$ and $\ai{0}_i = -1$. In essence, these filters are ``good'' detectors because they detect the discriminative patterns, with the right sign of $\ai{0}_i$.

In the second part we analyze the dynamics of the filters in the first layer. We will show that the ``lucky'' filters continue to detect the discriminative patterns and their projection on either $\oo_1$ or $\oo_2$ becomes larger in each iteration. In contrast, we upper bound the norm of the filters that are "non-lucky". Thus, after training the first layer, $\layeralg$ creates a new representation of the data in the second layer with the following properties: there are sufficiently many discriminative features with sufficiently large absolute values, and the remaining features have a bounded absolute value.

In the third part, we analyze the optimization of the second layer on the new representation. Using the properties of the representation, proved in the second part, we show that this representation induces a distribution on the samples which is linearly separable. Furthermore, it can be classified with margin 1 by a linear classifier of low norm. Then, we apply a result of \citet{soudry2018implicit}, which implies that training the second layer, which is equivalent to logistic regression on the new representation, converges to a low norm solution with zero training loss. Finally, we apply a norm-based generalization bound \citep{shalev2014understanding} to obtain the sample complexity guarantee.

\underline{\textbf{Part 1: Properties of the Initialization}:} 

Define the sets $\cA^+ = \left\{i \mid \ai{0}_i = 1 \right\}$, $\cA^- = \left\{i \mid \ai{0}_i = -1 \right\}$ and the following sets:

\begin{align*}
\label{eq:disc_filters}
\cW^+_t &= \left\{i \mid \argmax_{l \in \cO \setminus \{2\}} \wvec{t}_i\cdot \oo_l = 1, \, \wvec{t}_i\cdot \oo_1 > 0\right\} \\
\cW^-_t &= \left\{i \mid \argmax_{l \in \cO \setminus \{1\}} \wvec{t}_i\cdot \oo_l = 2,\, \wvec{t}_i \cdot \oo_2 > 0 \right\}
\numberthis
\end{align*}

The sets $\cW^+_0\cap \cA^+$ and $\cW^-_0 \cap \cA^-$ correspond to the set of ``lucky'' filters. We prove a lower and upper bound on the size of these sets.
\begin{lem}
	\label{LEM:INIT_GOOD_FILTERS}
	With probability at least $1-4e^{-d}$:
	\begin{equation}
	    \frac{k}{4d} \le \left|\cW^+_0\cap \cA^+\right|, \left|\cW^-_0 \cap \cA^-\right| \le \frac{k}{d}
	\end{equation}
\end{lem}

The proof uses the fact that $\prob\left(i \in \cW^+_0\cap \cA^+\right) = \frac{\left(1-2^{-d+1}\right)}{2(d-1)}$. Then, by concentration of measure for $k \ge d^3$, roughly $\frac{k}{2d}$ filters will be in $\cW^+_0\cap \cA^+$. The same argument holds for $\cW^-_0\cap \cA^-$. The proof is given in Section \ref{sec:thm_part1}.

\underline{\textbf{Part 2: First Layer Dynamics}:}

The following lemma shows the dynamics of the ``lucky'' neurons that detect the positive patterns.
\begin{lem}
\label{LEM:DYNAMICS_POS}
For all $0 \le t \le T_1$ and all $i \in \cW^+_0 \cap \cA^+$ the following holds:
\begin{enumerate}
    \item $\oo_1 \cdot \wvec{t}_i \ge \frac{t\eta_1}{9}$.
    \item For all $j \neq 1$, it holds that $\oo_j \cdot \wvec{t}_i \le r$.
\end{enumerate}
Furthermore, for all $\xx_+ \in S_1^+$, $\pwvec{t}_i(\xx_+) = \oo_1$.
\end{lem}

The lemma shows that the projection of the filter on $\oo_1$ grows significantly, while the projection on other $\oo_i$ remains small. Finally, it shows that for any positive point in $S_1$, the pattern which maximally activates the filter is $\oo_1$. Thus, the filter is correctly detecting the positive pattern. The proof is technical and shows that the properties above hold by induction on $t$. It is given in Section \ref{sec:proof_dynamics_first_layer}.

By the symmetry of our setting we get by Lemma \ref{LEM:DYNAMICS_POS} a similar result for the ``lucky'' neurons that detect negative patterns.

\begin{cor}
\label{cor:neg}
With probability at least $1- 4e^{-d} - 4e^{-\frac{m}{36}}$, for all $0 \le t \le T_1$ and all $i \in \cW^-_0 \cap \cA^-$ the following holds:
\begin{enumerate}
    \item $\oo_2 \cdot \wvec{t}_i \ge \frac{t\eta_1}{9}$.
    \item For all $j \neq 2$, it holds that $\oo_j \cdot \wvec{t}_i \le r$.
\end{enumerate}
Furthermore, for all $\xx_- \in S_1^-$, $\pwvec{t}_i(\xx_-) = \oo_2$.
\end{cor}

Finally, we provide a simple bound on the output of all neurons (including the "non-lucky" ones).

\begin{lem}
\label{lem:second_layer__features_norm}
For all $1 \le t \le T_1$, $1 \le i \le k$, $1 \le j \le d$ and $\xx$ sampled from $\cD$, it holds that $\xx[j] \cdot \wvec{t}_i \le 2 \eta_1 t$.
\end{lem}

The proof is given in Section \ref{sec:proof_simple_bound}.

\underline{\textbf{Part 3: Optimizing the Second Layer}:}

We conclude the proof of the theorem by analyzing the optimization of the second layer. Here we sketch the analysis and defer the details to Section \ref{sec:proof_part3}.

For each $\xx$ sampled from $\cD$, we define $\zz(\xx) \in \reals^k$ such that for all $1 \le i \le k$, its $i$th entry is $\zz(\xx)_i = \max_j\left\{\sigma\left(\wvec{T_1}_i\cdot \xx[j] \right)\right\}$ (namely, these are the values of the output of the pooling of each channel, which serve as features for the second layer). Then, we define a new distribution of points $\cD_{\zz}$ over $\reals^k \times \{\pm 1\}$, which samples a point $(\zz(\xx),y)$ where $(\xx,y) \sim \cD$.

Using the results of the first layer dynamics, we show that $\cD_{\zz}$ is linearly separable and can be separated with margin 1 by a classifier $\vv$ with $\left\|\vv\right\| = O\left(\sqrt{\frac{d}{k}}\right)$. Then, we use recent results on logistic regression \citep{soudry2018implicit}, to show that by optimizing the second layer, $\layeralg$ will converge to a low norm solution with zero training loss. Finally, we apply norm-based generalization bounds \citep{shalev2014understanding}. Since for all $\xx$, $\left\|\zz(\xx)\right\| = O(\sqrt{k})$, we obtain a sample complexity guarantee for $\layeralg$ of order $O\left(\left\|\vv\right\|^2 \max_{\xx}\left\|\zz(\xx)\right\|^2\right) = O\left(d\right)$.
\end{proof}

\begin{figure*}[h]
	\begin{subfigure}{.32\textwidth}
		\centering
		\includegraphics[width=1.0\linewidth]{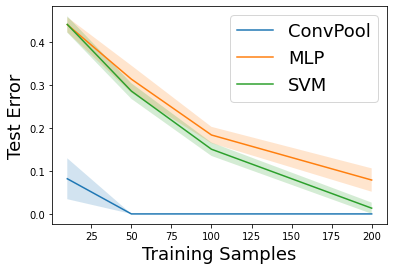}
		\caption{}
		\label{fig:sample_size_toy_small_eps}
	\end{subfigure}%
	\begin{subfigure}{.33\textwidth}
		\centering
		\includegraphics[width=1.0\linewidth]{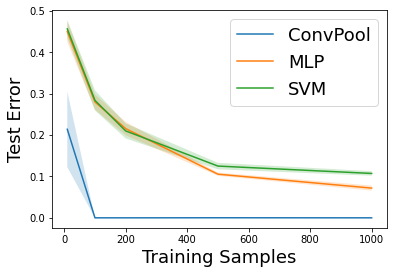}
		\caption{}
		\label{fig:sample_size_toy_large_eps}
	\end{subfigure}%
	\begin{subfigure}{.33\textwidth}
		\centering
		\includegraphics[width=1.0\linewidth]{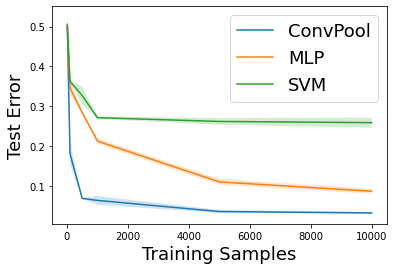}
		\caption{}
		\label{fig:mnist_exps}
	\end{subfigure}%

	\caption{Empirical evaluation of ConvPool architecture to different baselines. (a) Test error as a function of train sample size, for data sampled according to our pattern detection distribution. Note: this data is linearly separable in $\xx$ by construction, and we verified that all models had zero training error. (b) Test error as a function of train sample size, for data that as in (a), but with added noise vector $\vv$ with $\left\|\vv\right\| \le 1$. The resulting data is not linearly separable. Note: we verified that for the linear model training error was non-zero in many of the cases. For the other models training error was zero. (c) Results on an MNIST pattern detection problem. }
	\label{fig:psi_exps}
\end{figure*}

\section{Comparison with FCNs}
In the previous section we showed that overparameterized CNNs have good sample complexity for learning the pattern distributions $\cD$ in Section \ref{sec:problem_formulation}. How do overparameterized fully connected networks compare with CNNs in our setting?  To address this question, we apply recent results of \citet{brutzkus2018sgd}. They provide generalization guarantees for one-hidden layer overparameterized fully connected networks on linearly separable data. We will show that their bound for FC networks can be $O(d^{2r})$ for any $r \ge 1$. In contrast, Theorem \ref{thm:conv} shows a generalization bound for CNNs which is linear in $d$. We note that to fully demonstrate a gap between the methods we also need a lower bound on the FCN for the distribution $\cD$, and we leave this for future work. Nonetheless, we show empirically, that these generalization bounds predict the performance gap between CNNs and FCNs in our setting.

\commenteps{
First, we show that for a sufficiently small $\epsmax$, $\cD$ is linearly separable. Therefore, in this case we can compare their generalization guarantees with ours for CNNs. The pro

\begin{prop}
If $\epsmax < \frac{1}{2n}$ then $\cD$ is linearly separable.
\end{prop}
}

We begin by noting that the distribution $\cD$ is linearly separable in $\xx$, because one can set $\ww\in\reals^{n d}$ to be a concatenation of $n$ copies of the pattern difference $\oo_1-\oo_2$ and because of orthogonality this will correctly classify the data. We next explain how \citet{brutzkus2018sgd} can be used to obtain a sample complexity bound for learning this data with a fully connected leaky ReLU net.

Assume that $\cD$ is linearly separable with margin 1 by a classifier $\ww^*$, i.e., for all $(\xx,y) \in \cD$, $y\ww^*\cdot \xx \ge 1$. In \citet{brutzkus2018sgd} they consider the following fully connected network:
\begin{equation}
\label{eq:fc_network}
\fcth{\thth}(\xx)=
\sum_{i=1}^{k}a_i\psi\left(\ww_i\cdot \xx\right)
\end{equation}

for $\thth = (W,\aa)$ where in our setting $\ww_i \in \reals^{nd}$ is the $i$th row of $W \in \reals^{k\times nd}$, $\aa \in  \reals^k$ and $\xx \in \reals^{nd}$. $\psi(x) = \max\{\alpha x,x\}$ is the Leaky ReLU activation. 

They show that SGD converges to a zero training error solution with sample complexity of $O(\left\|\ww^*\right\|^2 R^2)$, where $R$ is the maximum norm of the data, $R=\max_{\xx}{\left\|\xx\right\|}$. In our setting it holds that $R^2 = n$ (because each point $\xx$ consists of $n$ patterns, each of norm $1$). Importantly, this bound is independent of the network size $k$. 

We note that the bound $O(\left\|\ww^*\right\|^2 R^2)$ also holds for the hard-margin linear SVM \citep{shalev2014understanding}. Therefore, our following conclusions hold for this algorithm as well. In the next section we show experiments that compare CNNs, FCNs and SVMs in our setting and  corroborate our findings.

The generalization bound of $O(\left\|\ww^*\right\|^2 R^2)$ holds for \textit{any} $\ww^*$ which separates with margin 1. Thus, the best bound can be achieved with $\ww^*$ that has the lowest norm and separates the data with margin 1. Next we show that the lowest norm is at least $\sqrt{n}$.

\begin{prop}
\label{prop:linear_norm}
Define
\begin{equation}
    \hat{\ww} = \argmin_{\ww \in \reals^{nd}}\left\|\ww\right\|^2\,\,\text{s.t.}\,\,\,\,\forall (\xx,y)\sim \cD \,\,\,\,\,y\ww \cdot \xx \ge 1
\end{equation}
Then $\left\|\hat{\ww}\right\|^2 \ge  n$.
\end{prop}
\begin{proof}
Assume by contradiction that $\left\|\hat{\ww}\right\|^2 = \sum_{1 \le i \le n} \left\|\hat{\ww}[i]\right\|^2 < n$. Then, there exists $1 \le i \le n$ such that $\left\|\hat{\ww}[i]\right\|^2 < 1$. Define a positive point $(\xx_+, 1)$ such that $\xx_+[i] = \oo_1$ and $\xx_+[j] = \oo_3$ for $j\neq i$. Similarly, define a negative point $(\xx_-, -1)$ such that $\xx_-[i] = \oo_2$ and $\xx_-[j] = \oo_3$ for $j\neq i$. Then it holds that:
\begin{align*}
\label{eq:pos_fc}
\hat{\ww} \cdot \xx_+ &= \sum_{1 \le j \le n}\hat{\ww}[j]\xx_+[j] = \oo_1 \cdot \hat{\ww}[i] + \sum_{j \neq i}\hat{\ww}[j]\cdot \oo_3  \ge 1 \numberthis        
\end{align*}
and similarly
\begin{align*}
\label{eq:neg_fc}
\hat{\ww} \cdot \xx_- = \oo_2 \cdot \hat{\ww}[i] + \sum_{j \neq i}\hat{\ww}[j]\cdot \oo_3  \le -1 \numberthis        
\end{align*}
By subtracting \eqref{eq:neg_fc} from \eqref{eq:pos_fc} we get:
\begin{equation}
\label{eq:contradiction}
    \hat{\ww}[i] \cdot (\oo_1 - \oo_2) \ge 2
\end{equation}
but since $\left|\hat{\ww}[i] \cdot (\oo_1 - \oo_2)\right| \le 2\left\|\hat{\ww}[i]\right\|$, we have by \eqref{eq:contradiction}  $\left\|\hat{\ww}[i]\right\| \ge 1$, which is a contradiction.
\end{proof}

Proposition \ref{prop:linear_norm} implies that the best possible bound of \citet{brutzkus2018sgd} for FC networks, or margin bound for linear SVM is $O(n^2)$ in our setting. Thus for $n = \Theta(d^r)$, $r\ge 1$ the bounds for FC networks and linear SVM are $O(d^{2r})$. In contrast, Theorem \ref{thm:conv} shows a generalization guarantee for CNNs of $O(d)$ for any $n$. This gap suggests that CNNs should significantly outperform FCNs and linear SVM in our setting. Next, we provide empirical evidence for this.

\section{Experiments}
\label{sec:exps}
In this section we provide empirical evaluation of learning with our pooling architecture and compare it to several other models. As baselines we consider:
\begin{itemize}
    \item {\bf{ConvPool}}: Our convolution and max-pooling model in \eqref{eq:network}. We verified that layer-wise training performs very similarly to standard  training, and thus we report results on standard training with Adam \citep{kingma2014adam} in what follows.
    \item {\bf{MLP}}: A standard fully connected neural network with one hidden layer. The network receives the complete $\xx$ as input (with all patterns). We use a number of hidden neurons that results in the same number of parameters as {\bf ConvPool}.
    \item {\bf{SVM}}: A hard-margin linear SVM with $\xx$ as input. This will return zero training errors only when the data is linearly separable. This is the case for our distribution $\cD$, but no longer the case when we add noise to the patterns (see below). 
\end{itemize}
All experiments used a test set of size $1000$, and were repeated $5$ times with mean and std reported on figures.

We begin with a toy data setting. We created data for a detection problem where all $\oo\in\mathbb{R}^{20}$ vectors were uniformly sampled from the rows of a uniformly sampled orthogonal matrix and $n=10$. ConvPool used 500 channels. 
Figure \ref{fig:sample_size_toy_small_eps} shows results for this setting. ConvPool can be seen to outperform the other methods. In Figure \ref{fig:sample_size_toy_large_eps} we go beyond our analyzed setting, and add independent random noise $\vv$ to each pattern where $\left\|\vv\right\| \le 1.0$. This makes the problem non linearly-separable. As expected, the linear method now fails, but ConvPool performs well and outperforms MLP.

Next, we consider the effect of the number of patterns $n$ on performance. As shown in Proposition \ref{prop:linear_norm}, the norm of the max-margin linear classifier is lower bounded by $n$. Thus, increasing $n$ is expected to result in worse performance for MLP and SVM by the results in the previous section. In Figure \ref{fig:nun_patterns}, we vary the number of patterns, and indeed observe that performances of MLP and SVM deteriorate while that of ConvPool is only mildly affected (we used the same parameters as above and  noise level $\left\|\vv\right\| \le 1$). 

Finally, we evaluate on the MNIST data set. We create a detection problem as in \figref{fig:mnist} where the discriminative patterns are the digits three and five and the spurious patterns are all other digits. Each input image contains four patterns (i.e., four digits). We used a relatively small number of patterns to make the problem not linearly separable for moderate sample sizes. We trained a 3 layer convolutional network as in \eqref{eq:network} with 500 channels. Results in \figref{fig:mnist_exps} again show excellent performance of the pooling model compared to the baselines.

\begin{figure}[t!]
    \centering
		\includegraphics[width=1.0\linewidth]{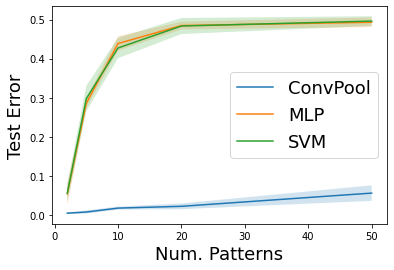}
	\caption{The effect of changing $n$, the number of patterns per image. It can be seen that this deteriorates the performance of the other methods while only mildly affecting the pooling model.}
	\label{fig:nun_patterns}
\end{figure}

\ignore{
Compare: SGD on Conv, Layerwise on Conv, NTK, Linear on Full, ReLU on Full 
Make sure to monitor training error.
\begin{itemize}
    \item Change number of samples (keep the layerwise Conv num channels very big so we are in over-parameterized)
    \item Change number of channels (want to show that generalization is not affected)
    \item Change epsilon (want to show breakdown at some epsilon level. breakdown in expressivity, optimization or generalization?)
\end{itemize}

Figures:

\begin{itemize}
    \item Figure 1: compare test error of conv (with Adam), MLP, SVM on linearly separable data (epsilon = 0). With mean and std.
    \item Figure 2: compare test error conv and MLP on nonlinear data. $\epsmin > \frac{2}{n}$. With mean and std.
    \item Figure 3: 
\end{itemize}
}

\section{Discussion}
In this paper we presented the first analysis of a convolutional max-pooling architecture in terms of optimization and generalization under over-parameterization. Our analysis
is for a natural setting of a detection problem where certain patterns ``identify'' the class and the others are irrelevant. Our analysis predicts a significant performance gap between CNNs and FCNs, which we observe in experiments.

While our analysis is the first step towards understanding pattern detection architectures, many open problems remain. The first is extending the pattern structure from orthogonal patterns to more general distributions. For example, we can consider the discriminative pattern to be a combination of patterns across the image (e.g., the class of the image is positive only if certain multiple patterns appear in the image). Second, it would be interesting to extend the convolution so that there are overlaps between filters (although this is known to generate local optima even for simpler settings \citep{brutzkus2017globally}). Finally, a challenging extension is to a multi-layer architecture with repeated application of pooling. 

\section*{Acknowledgements}
This research is supported by the European Research Council (ERC) under the European
Unions Horizon 2020 research and innovation programme
(grant ERC HOLI 819080). AB is supported by the Google Doctoral Fellowship in Machine Learning. 



\bibliography{pattern_detection}
\bibliographystyle{icml2020}

\newpage
\onecolumn
\appendix

\clearpage

\section{Convergence Rates for Theorem 5.1}
\label{sec:convergence_rates}

In \citet{ji2019refined}, Theorem 4.2, they show the following for logistic regression initialized at zero and a certain learning rate schedule. The margin of the learned classifier is $\frac{\gamma}{2}$ where $\gamma$ is the max-margin after $O\left(\frac{1}{\gamma^2}\right)$ iterations.\footnote{$O$ hides a dependency on $\log{m}$.} They show this for normalized points with norm 1. In our case (see the proof of Theorem \ref{thm:conv}), the max margin after normalizing the points to have norm 1, is $\frac{1}{\sqrt{d}}$. Thus, under their assumptions, after $O(d)$ iterations we converge to a solution whose margin is a $\frac{1}{2}$-multiplicative approximation of the max margin. Therefore, we obtain for this solution, up to a constant, the same generalization guarantees as the max margin classifier (which we provide in the theorem).

\section{Proof of Lemma \ref{LEM:INIT_GOOD_FILTERS}}
\label{sec:thm_part1}

By definition of the initialization we have $\prob\left(i \in \cA^+\right) = \frac{1}{2}$. Furthermore, we have that $\prob\left(i \in \cW^+_0\right) = \frac{\left(1-2^{-d+1}\right)}{d-1}$. This follows, since with probability $2^{-d+1}$, for all $\oo \in \cO\setminus\{2\}$, $\wvec{0}_i \cdot \oo \le 0$. On the other hand, with probability $\left(1-2^{-d+1}\right)$, there exists at least one $\oo \in \cO \setminus \{2\}$ such that $\wvec{0}_i \cdot \oo > 0$. Assume we condition on the latter event. Then, we get by symmetry that $\oo_1$ maximizes the dot product with $\wvec{0}_i$, among patterns in $\cO \setminus \{2\}$, with probability $\frac{1}{d-1}$. 

By independence of $W_0$ and $\aai{0}$, we have: $\prob\left(i \in \cW^+_0\cap \cA^+\right) = \frac{\left(1-2^{-d+1}\right)}{2(d-1)}$. Then, by Hoeffding's inequality we get:
\begin{align*}
    \prob\left(\left|\frac{\left|\cW^+_0\cap \cA^+\right|}{k} - \frac{\left(1-2^{-d+1}\right)}{2(d-1)}\right| > \frac{1}{4d}\right) &\le 2e^{-2k\left(\frac{1}{4d}\right)^2} \\ &\le 2e^{-d} \numberthis
\end{align*}
where in the last inequality we used the assumption on $k$.  Since $\frac{\left(1-2^{-d+1}\right)}{2(d-1)} \ge \frac{1}{2d}$ and $\frac{\left(1-2^{-d+1}\right)}{2(d-1)} \le \frac{1}{d}$ for $d \ge 3$, we get that with probability at least $1-2e^{-d}$, $\left|\cW^+_0\cap \cA^+\right| \ge \frac{\left(1-2^{-d+1}\right)k}{2(d-1)} - \frac{k}{4d} \ge \frac{k}{4d}$ and $\left|\cW^+_0\cap \cA^+\right| \le \frac{\left(1-2^{-d+1}\right)k}{2(d-1)} + \frac{k}{4d} \le \frac{k}{d}$. By the symmetry of our problem and definitions of the sets $\cW^+_0$, $\cW^-_0$, $\cA^+$, $\cA^-$, we similarly get that with probability at least $1-2e^{-d}$, $\frac{k}{4d} \le \left|\cW^-_0\cap \cA^-\right| \le \frac{k}{d}$. Applying the union bound concludes the proof.

\section{Proof of Lemma \ref{LEM:DYNAMICS_POS} }
\label{sec:proof_dynamics_first_layer}

We first prove the following two auxiliary lemmas.
\begin{lem}
\label{lem:filter_norm}
For all $0 \le t \le T_1$ and all $1 \le i \le k$, $\left\|\wvec{t}_i\right\| \le \eta_1(t+1)$. 
\end{lem}
\begin{proof}
First we notice that for all $1 \le i \le k$,  $\left\|\frac{\partial \cL_1}{\partial \ww_i}\left(W,\aai{0}\right)\right\| \le 1$.  This follows since for all $1 \le j \le n$ and all $\xx \in S_1$, $\left\|\xx[j]\right\| = 1$ (recall that $\left\|\oo\right\|=1$ for $\oo \in \cO$).

Therefore, for all $0 \le t \le T_1$ and $1 \le i \le k$, $\left\|\wvec{t}_i\right\| \le r + \eta_1 t \le  \eta_1(t+1)$. 
\end{proof}

\begin{lem}
\label{lem:net_bound}
For all $\xx \in S_1$ and $0 \le t \le T_1$ $\left|\cnnth{(\wmat{t},\ai{0})}(\xx)\right|\le \frac{1}{2}$.
\end{lem}
\begin{proof}
By Lemma \ref{lem:filter_norm} we have for all $\xx \in S_1$:

\begin{align*}
    \left|\cnnth{(\wmat{t},\ai{0})}(\xx)\right| &= \left|\sum_{i=1}^{k}\ai{0}_i\Big[\max_j\left\{\sigma\left(\wvec{t}_i\cdot \xx[j] \right)\right\}\Big]\right|\\ &\le k\max_{1 \le i \le k}\left\|\wvec{t}_i\right\|\max_{1 \le j \le n}\left\|\xx[j]\right\| \\
    &\le k\eta_1(t+1) \\ &\le \frac{1}{2}
\end{align*}

where the last inequality follows by the assumption on $\eta_1$.
\end{proof}

Lemma \ref{LEM:DYNAMICS_POS} follows by the following lemma.



\begin{lem}
With probability at least $1-4e^{-\frac{m}{36}}$, for all $0 \le t \le T_1$ and all $i \in \cW^+_0 \cap \cA^+$ the following holds:
\begin{enumerate}
    \item $\oo_1 \cdot \wvec{t}_i \ge \frac{t\eta_1}{9}$.
    \item For all $j \neq 1$, it holds that $\oo_j \cdot \wvec{t}_i \le r$.
\end{enumerate}
\end{lem}
\begin{proof}
We will prove the claim for $i \in \cW^+_0 \cap \cA^+$.  We prove the two claims by induction on $t$. In the proof by induction we also show a third claim that: for all $\xx_+ \in S_1^+$,  $\pwvec{i}_t(\xx_+) = \oo_1$.

For the proof, we condition on the event:
\begin{equation}
\label{eq:pos_neg}
\frac{\left|S_1^+\right|}{m_1},\frac{\left|S_1^-\right|}{m_1} \ge \frac{m_1}{3}    
\end{equation}
This holds with probability at least $1-4e^{-\frac{m}{36}}$ 
by applying Hoeffding's inequality and a union bound (over positive and negative samples).

For $t=0$, we have by definition for all $i \in \cW^+_0 \cap \cA^+$, $\oo_1 \cdot \wvec{t}_i > 0$. The second claim holds by the definition of the initialization. The third claim follows by the definition of $\cW^+_0 \cap \cA^+$.

Assume the three claims above hold for $t=T$. We will prove them for $t=T+1$. 

\underline{Proof of Claim 1.} 
By the gradient update in the first layer, the following holds for $i \in \cW^+_{0} \cap \cA^+$:
\begin{align*}
\label{eq:gd_update_first_layer}
    \wvec{T+1}_{i} &= \wvec{T}_{i} - \frac{\eta_1}{m_1} \sum_{\xx_+\in S_1^+}\ell'\left(\cnnth{(\wmat{T},\ai{0})}(\xx_+)\right)\pwvec{i}_T(\xx_+) \\ &+ \frac{\eta_1}{m_1}\sum_{\xx_-\in S_1^-}\ell'\left(-\cnnth{(\wmat{T},\ai{0})}(\xx_-)\right)\pwvec{i}_T(\xx_-) \numberthis
\end{align*}

where $l'(z) = -\frac{1}{1+e^z}$ is the derivative of the logistic loss. Note that for all $z$, $\left|\ell'(z)\right| \le 1$. 
Therefore, for all $\xx_- \in S_1^-$, we have:
\begin{equation}
\label{eq:derivative_neg}
    \left|\ell'\left(-\cnnth{(\wmat{T},\ai{0})}(\xx_-)\right)\right| \le 1
\end{equation}
By Lemma \ref{lem:net_bound} we have for all $\xx \in S_1$ $\left|\cnnth{(\wmat{t},\ai{0})}(\xx)\right| \le \frac{1}{2}$. Therefore, for all $\xx_+ \in S_1^+$:
\begin{equation}
\label{eq:derivative_pos}
\left|\ell'\left(\cnnth{(\wmat{T},\ai{0})}(\xx_+)\right)\right| \ge \frac{1}{1+\sqrt{e}} \ge \frac{1}{3}
\end{equation}

By the induction hypothesis, we have for $i \in \cW^+_0 \cap \cA^+$ and all $\xx_+ \in S_1^+$ that $\pwvec{i}_T(\xx_+) = \oo_1$. Therefore we have: 
\begin{equation}
\label{eq:dot_pos}
\pwvec{i}_T(\xx_+) \cdot \oo_1 = 1
\end{equation}

For all $\xx_- \in S_1^-$, we have $\pwvec{i}_T(\xx_-) = \oo_j$ for $j \neq  1$ that depends on $\xx_-$. Therefore:
\begin{equation}
\label{eq:dot_neg}
\pwvec{i}_T(\xx_-) \cdot \oo_1 = 0    
\end{equation}

By the facts above we complete the proof of the first claim:

\begin{align*}
\label{eq:disc}
\wvec{T+1}_{i} \cdot \oo_1 &\underset{\text{Eq. } \ref{eq:gd_update_first_layer},\ref{eq:derivative_neg},\ref{eq:derivative_pos}}{\ge} \wvec{T}_{i} \cdot \oo_1  + \frac{\eta_1}{3m_1}\sum_{\xx_+\in S_1^+}\pwvec{i}_T(\xx_+)\cdot \oo_1 \\ &-\frac{\eta_1}{m_1}\sum_{\xx_-\in S_1^-}\pwvec{i}_T(\xx_-)\cdot \oo_1 \\ &\underset{\text{Eq. } \ref{eq:pos_neg},\ref{eq:dot_pos},\ref{eq:dot_neg}}{\ge} \wvec{T}_{i} \cdot \oo_1 + \frac{\eta_1}{9} \\ &\ge \frac{(T+1)\eta_1}{9} \numberthis
\end{align*}
where the last inequality follows from the induction hypothesis.

\underline{Proof of Claim 2.} Since for all $\xx_+ \in S_1^+$, $\pwvec{i}_T(\xx_+) = \oo_1$ we have for all $1 \le j \le d$, $j \neq 1$: 
\begin{equation}
\label{eq:dot_pos_claim2}
\pwvec{i}_T(\xx_+) \cdot \oo_j = 0
\end{equation}

By the facts (1) for all $\xx_- \in S_1^-$ and $j \neq 1$ it holds that $\pwvec{i}_T(\xx_-) \cdot \oo_j \ge 0$ and (2) $l'(z) < 0$ for all $z$, we have:
\begin{equation}
\label{eq:dot_neg_claim2}
\frac{\eta_1}{m_1}\sum_{\xx_-\in S_1^-}\ell'\left(-\cnnth{(\wmat{T},\ai{0})}(\xx_-)\right)\pwvec{i}_T(\xx) \cdot \oo_j \le 0 
\end{equation}

Therefore we have for $j \neq 1$:
\begin{align*}
\label{eq:spur}
\wvec{T+1}_{i} \cdot \oo_j &\underset{\text{Eq.} \ref{eq:dot_pos_claim2},\ref{eq:dot_neg_claim2}}{\le} \wvec{T}_{i} \cdot \oo_j \le r \numberthis
\end{align*}
where the right inequality follows by the induction hypothesis.

\underline{Proof of Claim 3.} Since $r < \frac{\eta_1(T+1)}{9}$ we conclude by \eqref{eq:disc} and \eqref{eq:spur} that for all $\xx_+ \in S_1^+$,  $\pwvec{i}_{T+1}(\xx_+) = \oo_1$.
\end{proof}

\section{Proof of Lemma 5.5}
\label{sec:proof_simple_bound}
By Lemma \ref{lem:filter_norm}, for all $1 \le t \le T_1$ and $1 \le i \le k$, $\left\|\wvec{t}_i\right\| \le  \eta_1(t+1)$. Therefore, for all  $1 \le j \le d$ and $\xx$ sampled from $\cD$, $\xx[j] \cdot \wvec{t}_i \le 2 \eta_1 t$.

\section{Proof of Part 3 of Theorem 5.1}
\label{sec:proof_part3}

Here we condition on the events of previous lemmas which hold with probability at least $1- 4e^{-d} - 4e^{-\frac{m}{36}}$. For each $\xx$ sampled from $\cD$, define $\zz(\xx) \in \reals^k$ such that for all $1 \le i \le k$, its $i$th entry is $\zz(\xx)_i = \max_j\left\{\sigma\left(\wvec{T_1}_i\cdot \xx[j] \right)\right\}$. Notice that by \eqref{eq:max_pattern} we have $\zz(\xx)_i = \wvec{T_1}_i \cdot \pwvec{T_1}_i(\xx)$. Define a new distribution of points $\cD_{\zz}$ over $\reals^k \times \{\pm 1\}$, which samples a point $(\zz(\xx),y)$ where $(\xx,y) \sim \cD$.

Our goal is to show that $\cD_{\zz}$ is linearly separable and can be separated with a classifier of relatively low norm. Then, we will use recent results on logistic regression, which show that GD converges to low norm solutions. Therefore, by optimizing the second layer, $\layeralg$ will converge to a low norm solution. Finally, we will apply norm-based generalization bounds to obtain a generalization guarantee for $\layeralg$.

First we will show that $\cD_{\zz}$ is linearly separable. Indeed define $\vv^* \in \reals^k$ as follows. For $i \in \cW^+_{0} \cap \cA^+$ let $\vv^*_i = \frac{80d}{k \eta_1 T_1}$ and for $i \in \cW^-_{0} \cap \cA^-$ let $\vv^*_i = -\frac{80d}{k \eta_1 T_1}$. Set all other entries of $\vv^*$ to 0. Then for any $\zz(\xx_+)$  such that $(\xx_+,1) \sim \cD$, we have:
\begin{align*}
\zz(\xx_+) \cdot \vv^* &= \frac{80d}{k \eta_1 T_1}\sum_{i \in \cW^+_{0} \cap \cA^+}\wvec{T_1}_i \cdot \pwvec{T_1}_i(\xx_+) \\ &-\frac{80d}{k \eta_1 T_1}\sum_{i \in \cW^-_{0} \cap \cA^-}\wvec{T_1}_i \cdot \pwvec{T_1}_i(\xx_+) \\ &> \left(\frac{80d}{k \eta_1 T_1}\right) \left(\frac{k}{4d}\right)\left(\frac{\eta_1  T_1}{10}\right) \\ &- \left(\frac{80d}{k \eta_1 T_1}\right)\left(\frac{k}{d}\right)\left(\frac{\eta_1 T_1}{80}\right) \\ &= 1
\end{align*}
where the inequality follows by Lemma \ref{LEM:INIT_GOOD_FILTERS}, Lemma \ref{LEM:DYNAMICS_POS} and Corollary \ref{cor:neg}. By symmetry, we have $-\zz(\xx_-) \cdot \vv^* > 1$ for all $(\xx_-,-1) \sim \cD$.

Next, we proceed to apply Theorem 3 in \citet{soudry2018implicit}. It requires that $\eta_2 < 2\beta^{-1}\sigma_{\max}^{-2}\left(Z\right)m_2$,\footnote{We added the factor $m_2$ because \citet{soudry2018implicit} consider the empirical loss without dividing by the number of samples.} where $\beta$ is the smoothness parameter of the logistic loss, $Z \in \reals^{k \times m_2}$ is the matrix which contains $\zz(\xx_{i + \lceil\frac{m}{2}\rceil})$ in its $i$th column and $\sigma_{\max}(Z)$ is the maximum singular value of $Z$. In our setting, $\beta = 1$ and by Lemma \ref{lem:second_layer__features_norm} $\sigma_{\max}^2(Z) \le \left\|Z\right\|_F^2 \le 4m_2k\eta_1^2T_1^2 \le \frac{m_2}{4k}$. Thus, by our assumption $\eta_2 < 8k \le 2\sigma_{\max}^{-2}\left(Z\right)m_2$ holds.


Therefore, by this theorem we are guaranteed that:
\begin{equation}
    \lim_{t \rightarrow \infty }\frac{\aai{t}}{\left\|\aai{t}\right\|} = \frac{\hat{\aa}}{\left\|\hat{\aa}\right\|}
\end{equation}

where
\begin{equation}
    \hat{\aa} = \argmin_{\vv \in \reals^k}\left\|\vv\right\|^2\,\,\text{s.t.}\,\,\,\,\forall i \,\,\,\,\,y_i\vv \cdot \zz(\xx_i) \ge 1
\end{equation}

Specifically, gradient descent converges to zero training loss, i.e., $\lim_{T_2 \rightarrow \infty}\cL_2\left(\left(W_{T_1},\aa_{T_2}\right)\right) = 0$.

By optimality of $\hat{\aa}$ and Lemma \ref{LEM:INIT_GOOD_FILTERS} we have $\left\|\hat{\aa}\right\|^2 \le \left\|\vv^*\right\|^2 \le \frac{80^2 d^2}{k^2 \eta_1^2 T_1^2} \frac{2k}{d} = \frac{2 \cdot 80^2 d}{k \eta_1^2 T_1^2}$. Furthermore, $\left\|\zz(\xx)\right\|^2 \le 4k\eta_1^2 T_1^2$ by Lemma \ref{lem:second_layer__features_norm}. Therefore, we have $\left\|\hat{\aa}\right\|^2 \left\|\zz(\xx)\right\|^2 = O(d)$. Thus, by a standard margin generalization bound (e.g. Theorem 26.13 in \citet{shalev2014understanding} or \citet{bartlett2002rademacher}) we have with probability at least $1-\delta$:
\begin{align*}
    &\lim_{T_2 \rightarrow \infty}\prob_{(\xx,y) \sim \cD}\left(\textup{sign}\left(\cnnth{\left(\wmat{T_1},\aai{T_2}\right)}(\xx)\right)\neq y\right) \\ &= \prob_{(\xx,y) \sim \cD}\left(\textup{sign}\left(\cnnth{\left(\wmat{T_1},\frac{\hat{\aa}}{\left\|\hat{\aa}\right\|}\right)}(\xx)\right)\neq y\right) \\ &= O\left(\sqrt{\frac{d}{m}}\right)
\end{align*}
where $O$ hides an additive term which depends on $\delta$.

\begin{figure}[h]
    \centering
		\includegraphics[width=0.8\linewidth]{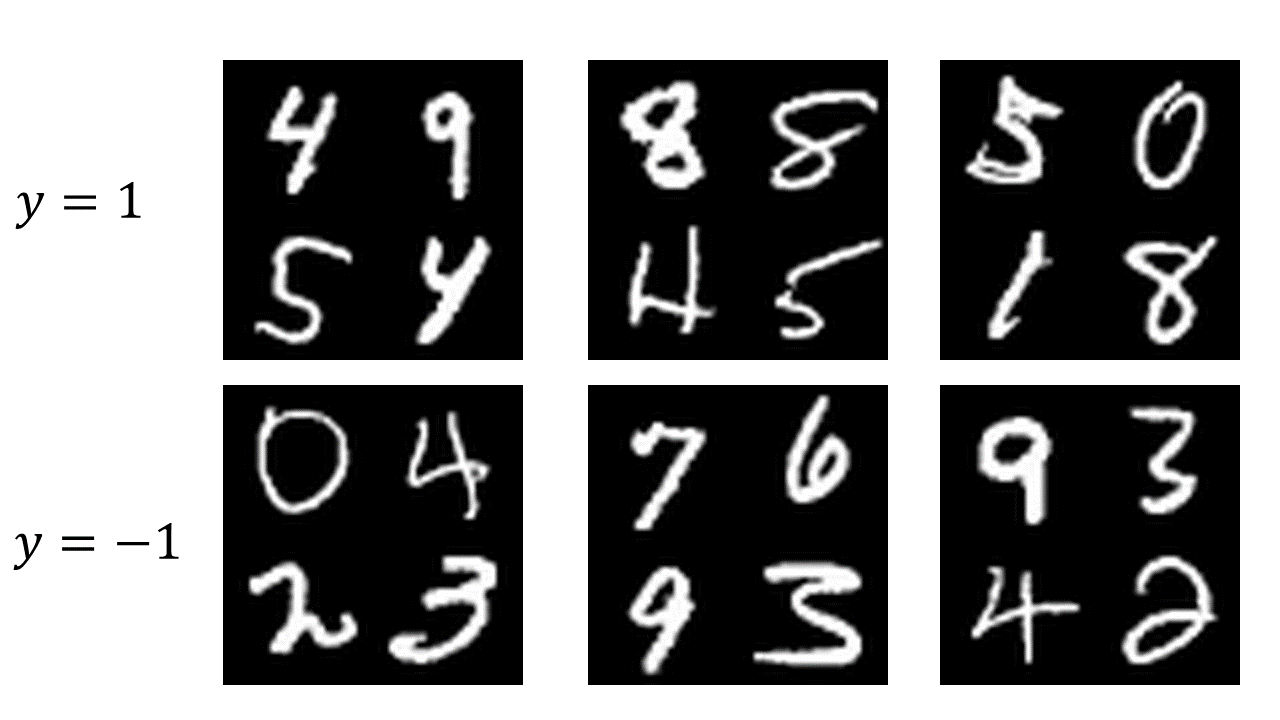}
	\caption{Data examples in the MNIST detection problem we experiment with in Section \ref{sec:exps}.}
	\label{fig:mnist}
\end{figure}



\end{document}